\documentclass[11pt]{article}

\usepackage{amsmath,amssymb,amsthm}
\usepackage{mathtools}
\usepackage{graphicx}

\usepackage{apacite}
\usepackage{natbib}
\usepackage{hyperref}

\usepackage{tikz}
\usetikzlibrary{positioning}
\usetikzlibrary{calc}

\usepackage{authblk}
\usepackage{float}
\usepackage{fullpage}

\newcommand{\calN}{\mathcal{N}}
\newcommand{\calC}{\mathcal{C}}

\newcommand{\calD}{\mathcal{D}}

\newcommand{\E}{\operatorname{\mathbb{E}}}
\newcommand{\He}{\textrm{He}}

\newcommand{\bbR}{\mathbb{R}}

\renewcommand{\le}{\left}
\newcommand{\ri}{\right}

\newcommand{\eps}{\varepsilon}

\newcommand{\tia}[3]{#1_{#2;#3}}
\newcommand{\til}[3]{#1^{(#3)}_{#2}}
\newcommand{\tial}[4]{#1_{#2;#3}^{(#4)}}
\newcommand{\ti}[2]{#1_{#2}}
\newcommand{\TIL}[3]{#1^{#2}_{(#3)}}

\newcommand{\tl}[2]{#1^{(#2)}}

\newcommand{\nl}[1]{\n}
\newcommand{\n}{n}

\newcommand{\w}{W}
\newcommand{\wi}[1]{\ti{\w}{#1}}
\newcommand{\wil}[2]{\til{\w}{#1}{#2}}

\newcommand{\bias}{b}
\newcommand{\biasi}[1]{\ti{\bias}{#1}}
\newcommand{\biasil}[2]{\til{\bias}{#1}{#2}}

\newcommand{\z}{z}
\newcommand{\zl}[1]{\tl{z}{#1}}
\newcommand{\zial}[3]{\tial{z}{#1}{#2}{#3}}
\newcommand{\zil}[2]{\til{\z}{#1}{#2}}
\newcommand{\zia}[2]{\tia{z}{#1}{#2}}
\newcommand{\zi}[1]{\ti{\z}{#1}}

\newcommand{\sial}[3]{\tial{\sigma}{#1}{#2}{#3}}
\newcommand{\sil}[2]{\til{\sigma}{#1}{#2}}
\newcommand{\sia}[2]{\tia{\sigma}{#1}{#2}}

\newcommand{\gl}[1]{\tl{G}{#1}}
\newcommand{\gil}[2]{\til{G}{#1}{#2}}

\newcommand{\GIL}[2]{\TIL{G}{#1}{#2}}

\newcommand{\hg}{\widehat{G}}

\newcommand{\hgil}[2]{\til{\hg}{#1}{#2}}
\newcommand{\hgl}[1]{\tl{\hg}{#1}}

\newcommand{\HGIL}[2]{\TIL{\hg}{#1}{#2}}

\newcommand{\dg}{\widehat{\Delta G}}

\newcommand{\dgil}[2]{\til{\dg}{#1}{#2}}
\newcommand{\dgl}[1]{\tl{\dg}{#1}}

\newcommand{\ig}{K}

\newcommand{\igil}[2]{\til{\ig}{#1}{#2}}
\newcommand{\igl}[1]{\tl{\ig}{#1}}
\newcommand{\Igil}[2]{\ti{(\tl{\ig}{#2})^{-1}}{#1}}

\newcommand{\IGIL}[2]{\TIL{\ig}{#1}{#2}}

\newcommand{\cb}{C_b}
\newcommand{\cw}{C_W}
\newcommand{\cbl}[1]{\tl{\cb}{#1}}
\newcommand{\cwl}[1]{\tl{\cw}{#1}}

\DeclarePairedDelimiter\norm{\lVert}{\rVert}
\DeclarePairedDelimiter\brak{\langle}{\rangle}
\DeclarePairedDelimiter\braces{\{}{\}}

\newtheoremstyle{noparens}{}{}{\itshape}{}{\bfseries}{.}{ }{\thmname{#1}\thmnumber{ #2}\thmnote{ \textmd{#3}}}
\theoremstyle{noparens}
\newtheorem{theorem}{Theorem}[section]
\theoremstyle{plain}
\newtheorem{lemma}{Lemma}[section]
\newtheorem{corollary}{Corollary}[section]
\newtheorem{conjecture}{Conjecture}[section]

\newcommand{\parbreak}{\vspace{7px}\noindent}

\begin{document}
\title{Wide Neural Networks as a Baseline for the Computational No-Coincidence Conjecture}

\author{John Dunbar}
\author{Scott Aaronson\thanks{Supported by Open Philanthropy for Dr. Scott Aaronson's grant `Computational Complexity Theory with AI Alignment.' Correspondence: johnjdunbar@utexas.edu and aaronson@cs.utexas.edu. }}
\affil{UT Austin}
\date{}

\maketitle

\begin{abstract}
    We establish that randomly initialized neural networks, with large width and a natural choice of hyperparameters, have nearly independent outputs exactly when their activation function is nonlinear with zero mean under the Gaussian measure: $\mathbb{E}_{z \sim \mathcal{N}(0,1)}[\sigma(z)]=0$. For example, this includes ReLU and GeLU with an additive shift, as well as tanh, but not ReLU or GeLU by themselves. Because of their nearly independent outputs, we propose neural networks with zero-mean activation functions as a promising candidate for the Alignment Research Center's computational no-coincidence conjecture---a conjecture that aims to measure the limits of AI interpretability.
\end{abstract}
\section{Introduction}
    In recent years, our ability to interpret and understand the inner workings of AI models has progressed significantly.
    White-box methods leveraging our inner understanding have shown potential in applications from detecting backdoors \citep{anthropic25_bio,goldowskydill25_deception} and steering behaviors \citep{anthropic24_mono,turner24_steering} to unlearning knowledge \citep{cloud24_gradrouting,zou24_circuitbreakers}.
    However, model internals have resisted theoretical analysis because the training process is opaque and arbitrary AI models can be cryptographically obfuscated in the worst case \citep{goldwasser24_backdoors,chrisiano25_backdoors}. Without a robust understanding, it's difficult to attest to the reliability of our tools and applications. One effort to build our understanding is the computational no-coincidence conjecture (see \citet{neyman25_nocoincidence}), a concrete, theoretical conjecture that aims to measure the average case ability of white-box methods.

    The conjecture asserts that for random circuits $C$ and an extremely rare property $P(C)$, there should exist short explanations that prove when the property is true with only a small false positive rate.
    Specifically, the conjecture uses random reversible circuits $C: \braces{0,1}^{3n} \to \braces{0,1}^{3n}$ and the rare property $P(C)$ that inputs ending in $n$ zeros never correspond to outputs ending in $n$ zeros. We know this property to be rare by a result of \citet{gay25_rrc} that proves that outputs of a random reversible circuit are nearly independent. The conjecture then asserts that there should exist a polynomial-time verifier $V(C, \pi)$, receiving a circuit $C$ and explanation $\pi$ as input, such that an acceptable explanation always exists when $P(C)$ is true, and no explanation exists on 99\% of random circuits. 
    Since $P(C)$ is satisfied by much less than 1\% of random reversible circuits, the explanations have some affordance for false positives.
    If the computational no-coincidence conjecture is true and an accurate verifier exists, it would be evidence that rare behaviors induce detectable structure that makes interpretability feasible. If the conjecture is false, however, it would be evidence that inscrutable circuits can be common, and we should expect them to exist inside our AI models.

    Random reversible circuits are a nice toy model, but our end goal is to understand the no-coincidence conjecture for neural networks. In this work, we make a first step towards that goal by establishing conditions when randomly initialized neural networks behave like random functions and have nearly independent outputs. 
    We propose these random-behaving neural networks as an alternative construction for circuits $C$, and propose the condition that a set of inputs never correspond to an all-negative output as the rare property $P(C)$. 

    Our results use the first order perturbative expansion and approximation from \citet{roberts22_pdlt} which ignores small-width, $O(1/\n^2)$ effects (where $\n$ is the network width). Our main result, \autoref{thm:mix}, establishes that with appropriate hyperparameters\footnotemark{} and as depth and width scale at a suitable rate, the probability distribution of neural network outputs over the randomness in its weights approaches a standard Gaussian distribution (and thus approaches zero correlation between outputs) if and only if its activation function $\sigma$ is nonlinear and has zero mean under the Gaussian measure: 
    \footnotetext{Our hyperparameters are a generalization of Kaiming initialization \citep{he15}. They are described further in \autoref{sec:cov}.}
\begin{equation}
    \E_{z \sim \calN(0,1)} [\sigma(z)] = 0.
\end{equation}
    The zero-mean criterion is satisfied for tanh but not ReLU or GeLU functions, agreeing with prior research that suggests in different ways that tanh networks are biased towards complex functions \citep{poole16}, and ReLU networks are biased towards simple functions \citep{hanin19,depalma19,teney24}. But the criterion can also be satisfied by any activation function with an additive shift, such as $\sigma(z)=\text{ReLU}(z) - 1/\sqrt{2 \pi}$.

\section{Problem Setup}\label{sec:arch}

    Our neural networks are described by the following standard recurrence:
\begin{equation}\label{eq:nn_recurrence}
\begin{aligned}
    \zil{i}{1} &= \biasil{i}{1} + \sum_{j=1}^{\nl{0}} \wil{i j}{1} \ti{x}{j}, \\
    \sil{i}{\ell} &= \sigma(\zil{i}{\ell}), \\
    \zil{i}{\ell} &= \biasil{i}{\ell} + \sum_{j=1}^{\nl{\ell-1}} \wil{i j}{\ell} \sil{j}{\ell-1} \text{ for } \ell > 1.
\end{aligned}
\end{equation}
    Here, $\sigma : \bbR \to \bbR$ is our activation function, $\n$ is our network width, and at every layer $\ell$, our networks are parameterized by the weight matrix $\braces{\wil{i j}{\ell}}_{i,j=1,\ldots,\n}$ and bias vector $\braces{\biasil{i}{\ell}}_{i=1,\ldots,\nl{\ell}}$. Instead of a dedicated output layer, we consider the preactivations at an arbitrary layer as the final outputs of our neural networks, and for outputs smaller than the network width, we can freely ignore the extra neurons.
    Our weights and biases are drawn from a Gaussian distribution with a mean of zero and variance set by constants $\cbl{\ell}$ and $\cwl{\ell}$:
\begin{equation}
\begin{aligned}
    \biasil{i}{\ell} &\sim \calN(0, \cbl{\ell}), \\
    \wil{i j}{\ell} &\sim \calN(0, \cwl{\ell} / \nl{\ell-1}).
\end{aligned}
\end{equation}
    A single input to our neural network is a vector $x \in \bbR^{\nl{0}}$, but we'll often need to consider a dataset with many inputs. We'll denote the dataset as $\calD = \braces{\tia{x}{i}{\alpha}}_{i=1,\ldots,\nl{0};\alpha=1,\ldots}$, where we have used the notation $\tia{x}{i}{\alpha} \in \bbR$ to denote index $i$ of the input vector $\alpha$. For preactivations, the notation $\zial{i}{\alpha}{\ell} \in \bbR$ represents the value before the activation function at neuron index $i$ for input $\alpha$ at layer $\ell$.

    All expectations will be taken over the random variables $\wil{i j}{\ell}$ and $\biasil{i}{\ell}$, so to denote general Gaussian expectations with covariance matrix $K \in \bbR^{m \times m}$, we will use bracket notation:
\begin{equation}\label{eq:braknotation}
    \brak{f(\zi{1}, \ldots, \zi{m})}_K = \int \frac{1}{\sqrt{|2 \pi K|}} \exp\le(-\frac{1}{2} \sum_{i,j=1}^m \ti{(K^{-1})}{i j} \, \zi{i} \zi{j}\ri) f(\zi{1}, \ldots, \zi{m}) \prod_{i=1}^m d \zi{i}.
\end{equation}
    We will drop the subscript for expectations over a standard Gaussian measure $\brak{f(z)} = \brak{f(z)}_I$.
    We will also later need the assumption that $\sigma$ is square integrable under the Gaussian measure:
\begin{equation}
    \brak{\sigma(z)^2} < \infty.
\end{equation}

\section{Near Gaussianity}

    At the infinite width-limit, it's well known that randomly initialized neural networks become Gaussian processes, and the network's behavior can be completely described by a covariance matrix $\igl{\ell}$ \citep{lee18}. 
    In this setting, demonstrating that a neural network's outputs are independent merely requires proving that the covariance between any two outputs decays to zero. But neural networks are never infinitely wide in practice. It turns out that at finite width, neural networks are nearly Gaussian, and their non-Gaussian components decay with width \citep{roberts22_pdlt}.
\begin{theorem}[\citep{roberts22_pdlt}]\label{thm:roberts}
    For a neural network as defined in \autoref{sec:arch} with random weights and a fixed set of inputs $\calD = \braces{\tia{x}{i}{\alpha}}$ (where $\tia{x}{i}{\alpha} \in \bbR$ is index $i$ of datapoint $\alpha$) and in the nondegenerate case when the covariance matrix is invertible, the preactivations at every layer are distributed as nearly Gaussian.
\begin{equation}\label{eq:pzlmd}
\begin{aligned}
    P(\zl{\ell} \mid \calD) \propto \exp\bigg(&-\frac{1}{2} \sum_{\alpha_1, \alpha_2 \in \calD} \ti{\left(\igl{\ell} + \frac{1}{n} \tl{H}{\ell}\right)^{-1}}{\alpha_1 \alpha_2} \, \sum_{i= 1}^{\nl{\ell}} \zial{i}{\alpha_1}{\ell} \zial{i}{\alpha_2}{\ell} \\
    &+\frac{1}{8n} \sum_{\alpha_1,\ldots,\alpha_4 \in \calD} \til{J}{\alpha_1 \alpha_2 \alpha_3 \alpha_4}{\ell} \sum_{i_1, i_2=1}^\n \zial{i_1}{\alpha_1}{\ell} \zial{i_1}{\alpha_2}{\ell} \zial{i_2}{\alpha_3}{\ell} \zial{i_2}{\alpha_4}{\ell}
    + O\le(\frac{1}{\n^2}\ri) \bigg).
\end{aligned}
\end{equation}
    Here, the terms $\igil{\alpha_1 \alpha_2}{\ell}$, $\til{H}{\alpha_1\alpha_2}{\ell}$, and $\til{J}{\alpha_1\alpha_2\alpha_3\alpha_4}{\ell}$ are tensors that depend on the inputs $\calD$, layer $\ell$, and activation function $\sigma$, and the $O(1/\n^2)$ term contains additional tensors and instances of $z$. The covariances $\igil{\alpha_1\alpha_2}{\ell}$ of this nearly Gaussian distribution are determined by the following equations:
\begin{equation}
    \E[\zial{i_1}{\alpha_1}{\ell} \zial{i_2}{\alpha_2}{\ell}] = \delta_{i_1 i_2} \igil{\alpha_1 \alpha_2}{\ell} + O\le(\frac{1}{n}\ri),
\end{equation} 
\begin{equation}
    \igil{\alpha_1 \alpha_2}{1} = \cbl{1} + \frac{\cwl{1}}{\nl{0}} \sum_{j=1}^{\nl{0}} \tia{x}{j}{\alpha_1} \tia{x}{j}{\alpha_2},
\end{equation}
\begin{equation}
\begin{aligned}\label{eq:gildef}
    \igil{\alpha_1 \alpha_2}{\ell} = \cbl{\ell} + \cwl{\ell} \brak{\sigma(\zi{\alpha_1}) \sigma(\zi{\alpha_2})}_{\igl{\ell-1}} \textnormal{ for } \ell > 1.
\end{aligned}
\end{equation}
\end{theorem}

\parbreak
    For completeness, we prove a simplified form of \autoref{thm:roberts} in \autoref{sec:robertsproof}. For the full proof, see \citet[Chapter~4]{roberts22_pdlt}.
    From this theorem, we see that the correlations between outputs is primarily governed by the width-independent covariance matrix $\igil{\alpha_1\alpha_2}{\ell}$, but this matrix depends nontrivially on the hyperparameters and activation function. Our main result proves when this matrix approaches the identity, and thus, when different outputs of the neural network approach independence.

\section{Decaying Covariances}\label{sec:cov}

    Since the Gaussian component of \eqref{eq:pzlmd} does not decay with width, the independence of our neural network's outputs depends on the Gaussian covariances decaying to zero. We now fix hyperparameters to analyze when this decay occurs.
    Consider the special case where input vectors are normalized to $\sqrt{n}$ and the hyperparameters $\cb$ and $\cw$ are set as follows:
\begin{equation}\label{eq:critical}
\begin{aligned}
    \frac{1}{\nl{0}} \sum_{i=1}^{\nl{0}} (\tia{x}{i}{\alpha})^2 &= 1, \\
    \cwl{1} &= 1, \\
    \cwl{\ell} &= \frac{1}{\brak{\sigma(z)^2}} \textnormal{ for } \ell > 1, \\
    \cbl{\ell} &= 0.
\end{aligned}
\end{equation} 
    We'll refer to these hyperparameters in combination with the input normalization as a critical tuning because they prevent the diagonal terms $\igil{\alpha \alpha}{\ell}$, the variance of a single input's preactivations, from growing or decaying at every layer.
\begin{equation}
\begin{aligned}
    \igil{\alpha \alpha}{1} &= \frac{1}{n} \sum_{i=1}^n (\tia{x}{i}{\alpha})^2 = 1, \\
    \igil{\alpha \alpha}{\ell} &= \frac{\brak{\sigma(\z)^2}_{\igil{\alpha \alpha}{\ell-1}}}{\brak{\sigma(\z)^2}} = 1 \text{ for } \ell > 1.
\end{aligned}
\end{equation}
    These settings are also not uncommon in practice. The weight variance $\cw$ is a generalization of Kaiming initialization \citep{he15} which was also designed with the intent of keeping variances stable.

    Since the covariances $\igil{\alpha_1 \alpha_2}{\ell}$ are a function of two inputs $\alpha_1, \alpha_2$ and have no dependence on the other inputs, it's sufficient to consider datasets with only two inputs. There, $\igl{\ell}$ is a $2 \times 2$ matrix and it becomes clear that, because the diagonal terms are constant ($\igil{\alpha \alpha}{\ell}=1$), the off-diagonal terms at any layer are a function of their counterparts at the previous layer, $\igil{\alpha_1\alpha_2}{\ell}=\calC(\igil{\alpha_1\alpha_2}{\ell-1})$ for some $\calC:\bbR\to\bbR$. 
    We will see that as we progress through the layers, the off-diagonal term approaches a fixed point based on the activation function $\sigma$.

\begin{theorem}\label{thm:mix}
    In a neural network as defined in \autoref{sec:arch} with random weights, critical hyperparameters from \eqref{eq:critical}, nonlinear activation function $\sigma$, and a set of fixed inputs $\calD$ containing no scalar multiples (i.e. $|\igil{\alpha_1 \alpha_2}{1}| \neq 1$ for $\alpha_1 \neq \alpha_2$), preactivations are distributed as follows with exponentially decaying covariance if and only the activation function has zero mean under the Gaussian measure, $\brak{\sigma(z)} = 0$.
\begin{equation}
    \begin{gathered}
    P(\zl{\ell} \mid \calD) \propto \exp\le(-\frac{1}{2} \sum_{\alpha_1, \alpha_2 \in \calD} \Igil{\alpha_1 \alpha_2}{\ell} \sum_{i = 1}^n \zial{i}{\alpha_1}{\ell} \zial{i}{\alpha_2}{\ell} + O\le(\frac{1}{n}\ri)\ri), \\
    \igil{\alpha \alpha}{\ell} = 1, \qquad \igil{\alpha_1 \alpha_2}{\ell} = \exp(-\Omega(\ell)) \textnormal{ for } \alpha_1 \neq \alpha_2.
    \end{gathered}
\end{equation}
    Conversely, when the activation function has a nonzero mean, the covariance does not decay to zero, $\lim_{\ell \to \infty} \igil{\alpha_1 \alpha_2}{\ell} > 0$.
\end{theorem}

\parbreak
    Although the width-dependent $O(1/\n)$ terms may be a function of $\ell$, which we have treated as a constant so far, there exists a suitable way\footnotemark{} to scale both $n$ and $\ell$ simultaneously such that the $O(1/n)$ term becomes less than exponentially small in $\ell$.
    \footnotetext{Under the assumption that the $O(1/n^2)$ term in \eqref{eq:pzlmd} behaves reasonably and remains small, the necessary scaling is exponential for most activation functions: $n = \exp(\Theta(\ell))$.}
\begin{corollary}
    For a neural network defined the same as in Theorem 4.1 with depth $\ell$ and sufficiently large width $n$, the probability distribution of outputs $P(\zl{\ell} \mid \calD)$ is within $\exp(-\Omega(\ell))$ total variation distance from a standard Gaussian distribution.
\end{corollary}
\begin{proof}
    This corollary follows from rewriting the diminishing components in \autoref{thm:mix} in terms of $\ell$ at sufficiently large width scaling. First, using the perturbed matrix inverse identity, we can simplify the expression for $\Igil{\alpha_1 \alpha_2}{\ell}$.
\begin{equation}
    K^{-1} = (I + \exp(-\Omega(\ell)))^{-1} = I + \exp(-\Omega(\ell)),
\end{equation}
    where $I$ is the identity matrix. We then have
\begin{equation}
\begin{aligned}
    P(\zl{\ell} \mid \calD) &\propto \exp\left(-\frac{1}{2} \sum_{\alpha \in \calD} \sum_{i=1}^{\n} (\zial{i}{\alpha}{\ell})^2 + \exp(-\Omega(\ell))\right), \\
    &= \frac{\exp\left(-\frac{1}{2} \sum_{\alpha \in \calD} \sum_{i=1}^{\n} (\zial{i}{\alpha}{\ell})^2 + \exp(-\Omega(\ell))\right)}
    {\int \exp\left(-\frac{1}{2} \sum_{\alpha \in \calD} \sum_{i=1}^{\n} (\zia{i}{\alpha})^2 + \exp(-\Omega(\ell))\right) \prod_{\alpha, i} \tia{dz}{i}{\alpha}}, \\
    &= \frac{1}{\sqrt{(2 \pi)^n}} \exp\left(-\frac{1}{2} \sum_{\alpha \in \calD} \sum_{i=1}^{\n} (\zial{i}{\alpha}{\ell})^2\right)(1 + \exp(-\Omega(\ell))).
\end{aligned}
\end{equation}
    In the last equality, we Taylor expanded around $\exp(-\Omega(\ell)) = 0$. In this form, the total variation distance from a standard Gaussian is clear: 
\begin{equation}
    \int \left(P(\zl{\ell} \mid \calD) - \frac{1}{\sqrt{(2 \pi)^n}} \exp\left(-\frac{1}{2} \sum_{\alpha \in \calD} \sum_{i=1}^{\n} (\zial{i}{\alpha}{\ell})^2\right)\right) \prod_{\alpha, i} \tia{dz}{i}{\alpha} = \exp(-\Omega(\ell)).
\end{equation}
\end{proof}

\parbreak
    Before proving \autoref{thm:mix}, let's build up a few useful properties of the covariances $\igil{\alpha_1 \alpha_2}{\ell}$. Let $\calC:\bbR \to \bbR$ be the map from the covariance at one layer to the covariance at the next. 
\begin{equation}\label{eq:cintegral}
\begin{gathered}
    \igil{\alpha_1 \alpha_2}{\ell+1} = \calC(\igil{\alpha_1 \alpha_2}{\ell}), \\
    \calC(\igil{\alpha_1 \alpha_2}{\ell}) = \cb + \cw \brak{\sigma(z_1) \sigma(z_2)}_\Sigma = \frac{\brak{\sigma(z_1) \sigma(z_2)}_\Sigma}{\brak{\sigma(z)^2}}, \\
    \Sigma = \begin{pmatrix} 1 & \igil{\alpha_1 \alpha_2}{\ell} \\ \igil{\alpha_1 \alpha_2}{\ell} & 1 \end{pmatrix}.
\end{gathered}
\end{equation}
    In \autoref{fig:relu} we plot $\calC$ for the ReLU and tanh activation functions. The covariance between preactivations after many layers is equal to the repeated application of $\calC$ on the initial covariance $\igil{\alpha_1 \alpha_2}{1}$, and we can see in the plots that for ReLU and tanh, the repeated application of $\calC$ appears to bring most initial covariances to a fixed point at 1 or 0. In \autoref{sec:fpproof} we will prove that this is not a coincidence and that all nonlinear activation functions bring covariances to a fixed point. Eventually, to prove \autoref{thm:mix} we will need to understand exactly when the fixed point is at 0.
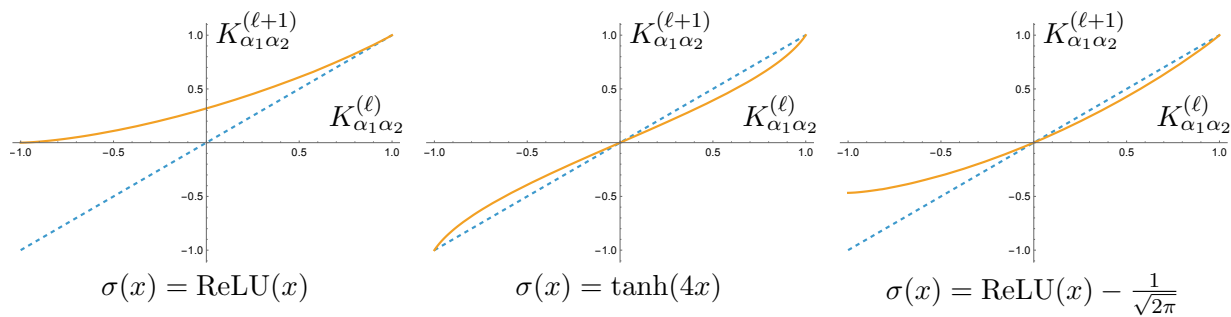
\begin{figure}
\centering
\begin{tikzpicture}
    \node at (-5.5,0) {\includegraphics[width=5.2cm]{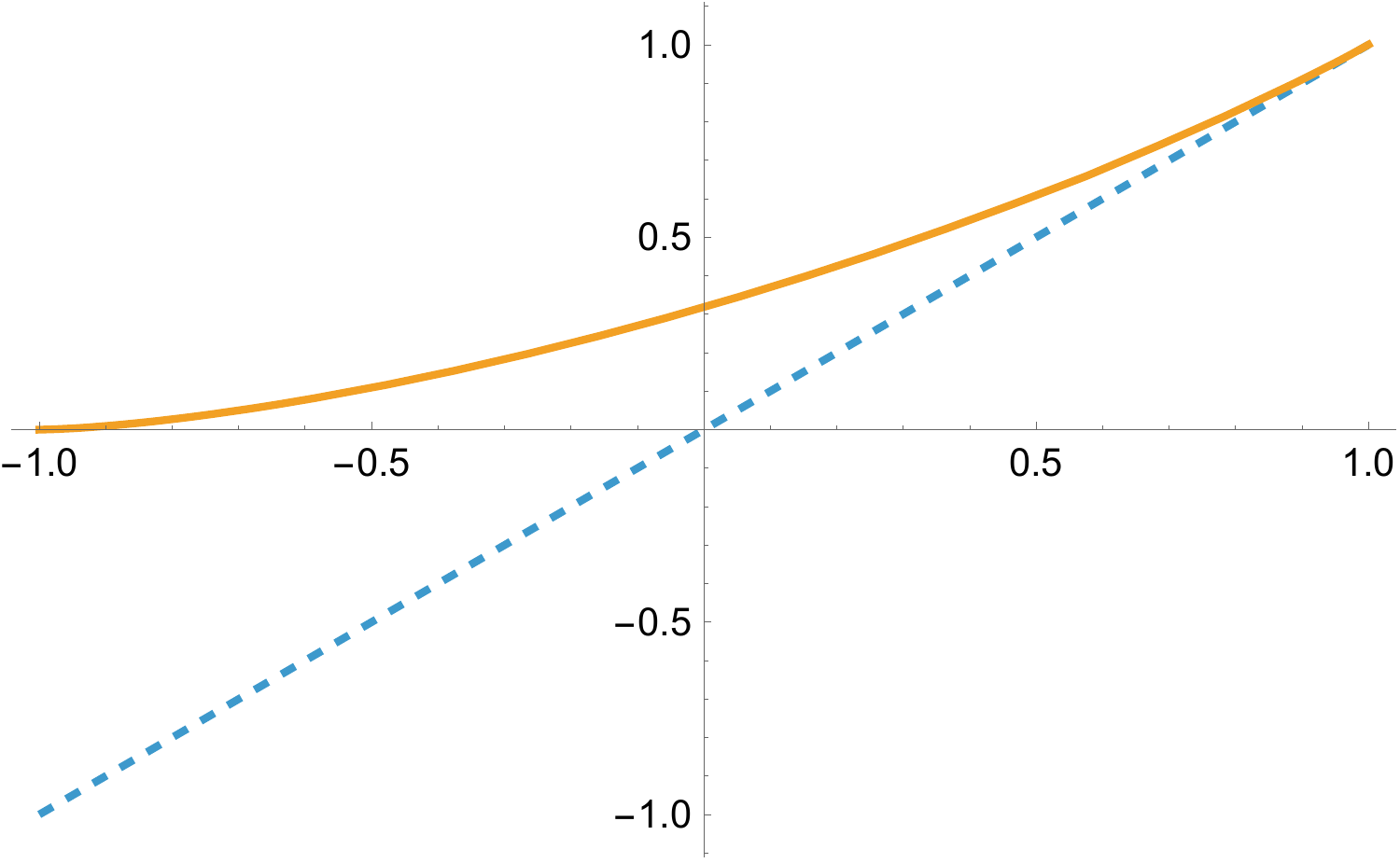}};
    \node at (0,0)    {\includegraphics[width=5.2cm]{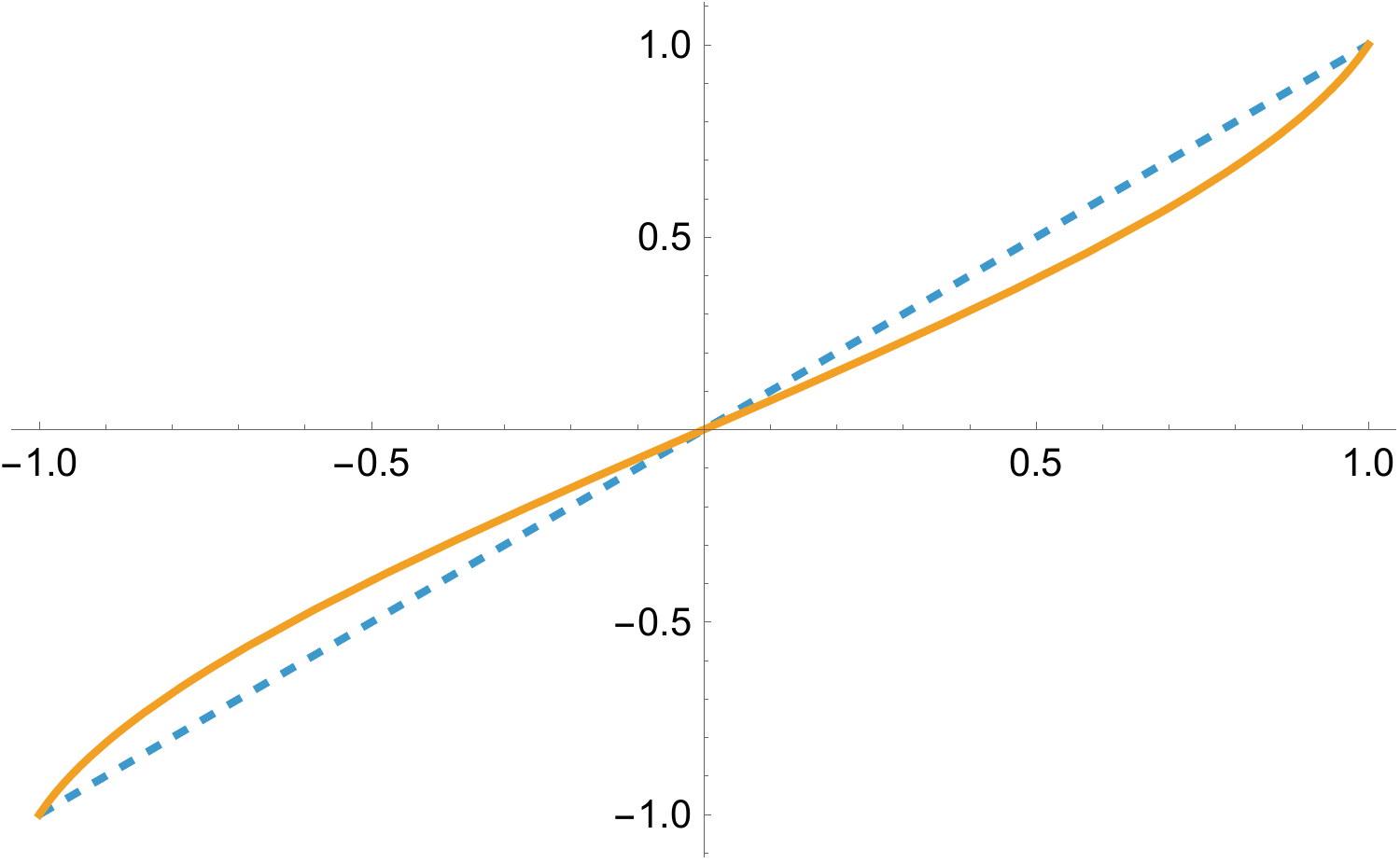}};
    \node at (5.5,0)  {\includegraphics[width=5.2cm]{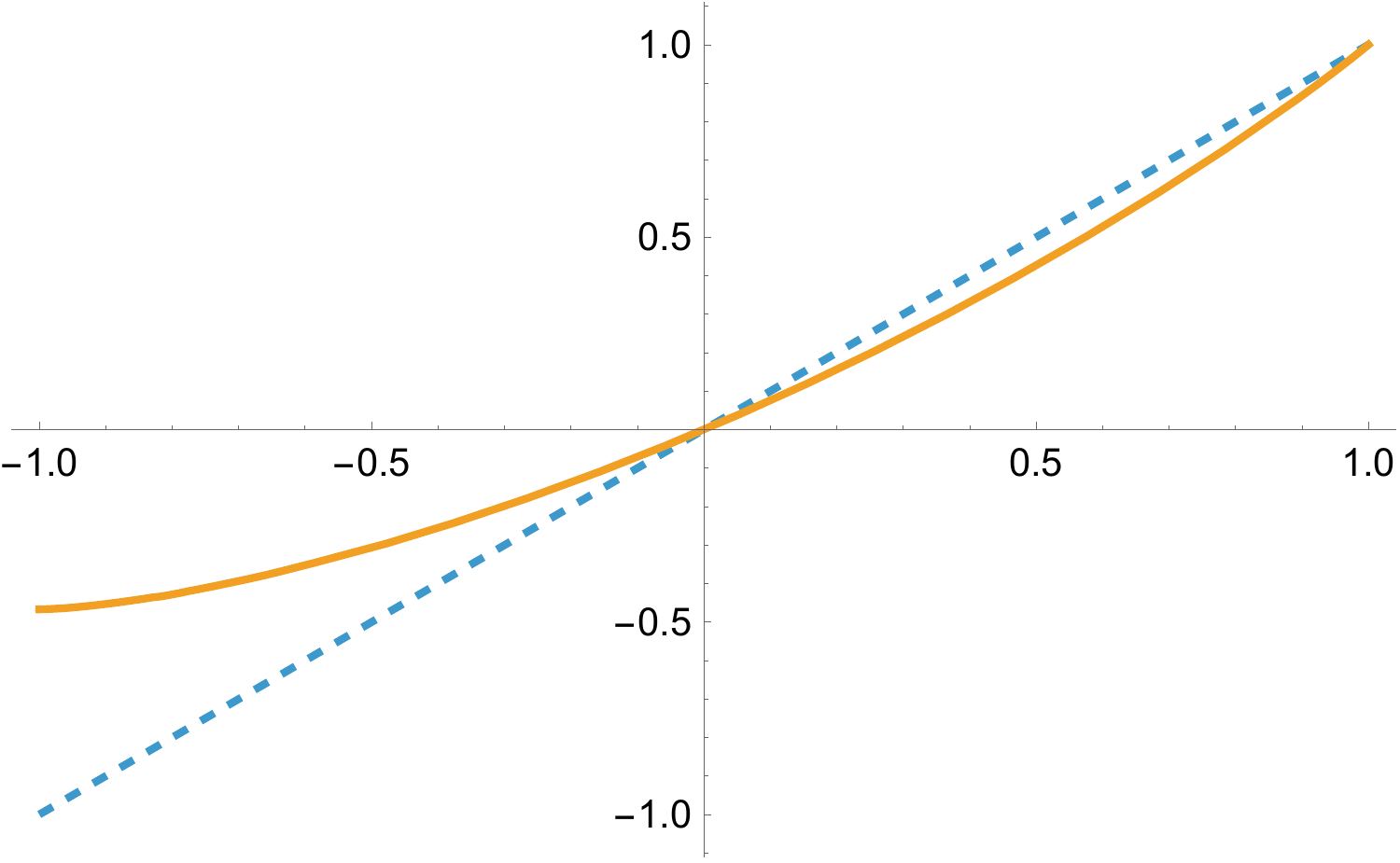}};
    \node[anchor=west] at (-5.5+0.0, 1.5) {$\igil{\alpha_1 \alpha_2}{\ell+1}$};
    \node[anchor=west] at (0   +0.0, 1.5) {$\igil{\alpha_1 \alpha_2}{\ell+1}$};
    \node[anchor=west] at (5.5 +0.0, 1.5) {$\igil{\alpha_1 \alpha_2}{\ell+1}$};
    \node[anchor=south east] at (-5.5+2.8, 0) {$\igil{\alpha_1 \alpha_2}{\ell}$};
    \node[anchor=south east] at (0   +2.8, 0) {$\igil{\alpha_1 \alpha_2}{\ell}$};
    \node[anchor=south east] at (5.5 +2.8, 0) {$\igil{\alpha_1 \alpha_2}{\ell}$};
    \node[anchor=north] at (-5.5, -1.6) {$\sigma(x)= \text{ReLU}(x)$};
    \node[anchor=north] at (0,    -1.6) {$\sigma(x) = \tanh(4x)$};
    \node[anchor=north] at (5.5,  -1.6) {$\sigma(x) = \text{ReLU}(x) - \frac{1}{\sqrt{2\pi}}$};
\end{tikzpicture}
\caption{Graphs of $\igil{\alpha_1 \alpha_2}{\ell+1} = \calC(\igil{\alpha_1 \alpha_2}{\ell})$ for when the activation function is ReLU (left), $\tanh(4x)$ (center), or a shifted ReLU (right). The $y=x$ line in dotted blue is included for comparison. For ReLU, the repeated application of $\calC$ brings initial points towards the stable fixed point at $\igil{\alpha_1 \alpha_2}{\ell}=1$. This means the preactivations on different inputs become more and more correlated as depth increases. For tanh (for which the $4x$ scaling was chosen to emphasize the effect and doesn't change it qualitatively), the repeated application of $\calC$ brings initial points towards 0 unless they start at $1$ or $-1$. This means preactivations on different inputs usually become less and less correlated as depth increases, but they stay identical if they started identical up to a sign. A similar effect occurs for a ReLU shifted to have zero mean under the Gaussian measure.}
\label{fig:relu}
\end{figure}

\subsection{Hermite Decompositions}

    To understand \eqref{eq:cintegral} and the behavior of $\calC$, we can decompose our activation function $\sigma$ into Hermite polynomials $\He_n$ which are orthogonal under Gaussian weighting \citep{bateman1955}. 
\begin{equation}
\begin{gathered}
\begin{aligned}
    \He_0(x) &= 1, \\
    \He_1(x) &= x, \\
    \He_2(x) &= x^2-1, \\
    \He_3(x) &= x^3-3x,
\end{aligned} \\
\begin{aligned}
    \brak{\He_n (z) \He_m (z)} &= \delta_{nm} n!, \\
    \norm{\He_n}^2 = \brak{\He_n(z)^2} &= n!.
\end{aligned}
\end{gathered}
\end{equation}
    Any function that is square integrable under the Gaussian measure ($\brak{\sigma(z)^2} < \infty$) can be decomposed into Hermite polynomials. This is a weak condition that, as mentioned before, we assume to be true of our activation functions.
\begin{align} 
    \sigma(z) &= \sum_{n=0}^{\infty} a_n \He_n (z), \\
    a_n &= \frac{1}{\norm{\He_n}^2} \brak{\sigma(z) \He_n (z)}.
\end{align}
    For correlated Gaussians, we can make a similar statement. The following identity derives from Mehler's equation. See \autoref{sec:mehlerproof} for the derivation.
\begin{equation}\label{eq:mehler}
\begin{gathered}
    \brak{\He_n(z_1) \He_m(z_2)}_\Sigma = \delta_{nm} \, n! \, k^n,  \\
    \Sigma = \begin{pmatrix} 1 & k \\ k & 1 \end{pmatrix}.
\end{gathered}
\end{equation}
After decomposing our activation function into Hermite polynomials, our map $\calC$ greatly simplifies into a polynomial. 
\begin{equation}\label{eq:cpoly}
\begin{aligned}
    \calC(k) &= \frac{\brak{\sigma(z_1) \sigma(z_2)}_\Sigma}{\brak{\sigma(z)^2}}, \\
    &= \frac{1}{\brak{\sigma(z)^2}} \brak*{
    \le(\sum_{n=0}^{\infty} a_n \He_n (z_1)\ri) 
    \le(\sum_{n=0}^{\infty} a_n \He_n (z_2)\ri)}_\Sigma, \\
    &= \frac{1}{\brak{\sigma(z)^2}} \sum_{n=0}^{\infty} (a_n)^2 \, n! \, k^n.
\end{aligned}
\end{equation}
Using this polynomial, we will prove that all nonlinear activation functions have an attractive fixed point.

\subsection{Fixed Points}\label{sec:fpproof}

    In this section we'll build up to and prove a lemma (\autoref{lem:fixpoint}) which states that the repeated application of $\calC$ always reaches a single fixed point for any initial covariance $\igil{\alpha_1 \alpha_2}{1} \in (-1,1)$. Once that is established, \autoref{thm:mix} immediately follows because our fixed point is at 0 exactly when $\calC(0) = 0$, and that is equivalent to when the first coefficient $a_0 = \brak{\sigma(z)}$ in \eqref{eq:cpoly} is 0. Although some lemmas require a nonaffine instead of nonlinear activation function, the affine, nonlinear case is simple to handle separately.

\begin{lemma}\label{lem:cconvex}
    The function $\calC$ is convex and non-negative in the $[0,1]$ domain, and strictly convex in that domain when the activation function $\sigma$ is nonaffine.
\end{lemma}
\begin{proof}
    The convexity follows because $\calC$ in \eqref{eq:cpoly} is a positive combination of convex monomials. For the strict convexity, note that $\He_0(x)=1$ and $\He_1(x)=x$, so Hermite decompositions with only the first two terms represent exactly the affine functions. Thus, the decomposition of any nonaffine function has some non-zero coefficient in $\{a_n: n > 1\}$ which multiplies a strictly convex monomial $\braces{k^n : n > 1}$ in \eqref{eq:cpoly}, causing $\calC$ to be strictly convex.
\end{proof}

\begin{lemma} \label{lem:pos}
    The function $\calC$ is above the diagonal, $\calC(k) > k$, in the $(-1,0)$ domain when the activation function $\sigma$ is nonlinear.
\end{lemma}
\begin{proof}
    First, we trivially have $\calC(1)=1$ because neural networks are deterministic, and identical activations stay identical. At the point $(1,1)$, our monomials in \eqref{eq:cpoly} are all 1 and we see that the coefficients $(a_n)^2 n! / \brak{\sigma(z)^2}$ are actually a convex combination.
    Thus in the $(-1,0)$ domain, $\calC$ is above the diagonal because it is a convex combination of monomials that are either at the diagonal (for the linear term) or above it. A monomial that's above the diagonal is always included in the combination when the activation function is nonlinear.
\end{proof}

\begin{lemma}\label{lem:fixpoint}
    For any starting input $k_0 \in (-1,1)$, the repeated application of $\calC$ on $k_0$ approaches a unique fixed point $k^* \in [0,1]$ when the activation function $\sigma$ is nonaffine.
\end{lemma}
\begin{proof}
    Let $\calC' = \frac{d}{dk} \calC(k)$ be the derivative of $\calC$. Let's split the proof into cases based on the value of the derivative at $k=1$.

\parbreak
\textit{Case 1:} If $\calC'(1) \leq 1$, then $\calC$ is always above the diagonal (except at $\calC(1)=1$). This is because \autoref{lem:pos} directly applies in the $(-1,0)$ domain and because of strict convexity in the $[0,1)$ domain. Thus, the repeated application of $\calC$ increases any initial value $k_0$ up to the fixed point at the domain's boundary, $\calC(1)=1$.
\begin{equation}
    \calC(k) > \calC'(1) (k-1) + 1 \geq k.
\end{equation}
    
\parbreak
    \textit{Case 2:} If $\calC'(1) > 1$, conversely, then there's a second fixed point $k^*$ in the domain $[0,1)$. 
    This is because \eqref{eq:cpoly} causes $\calC(0)$ to be at or above the diagonal, and the steep derivative $\calC'(1) > 1$ causes $\calC(1-\eps)$ for small $\eps$ to be below the diagonal. Above at one point and below at another, $\calC$ must be exactly on the diagonal somewhere in between. That fixed point is also unique because along with $\calC(1)=1$, a third fixed point would violate the strict convexity in the $[0,1)$ domain.

    Now that we've proven the existence of a fixed point $k^* \in [0,1)$, we must prove the convergence to it. In the $(-1,0)$ domain, \autoref{lem:pos} directly applies and the repeated application of $\calC$ increases values until they reach the $[0,1)$ domain. In the part of the $[0,1)$ domain where $\calC$ is below the diagonal (i.e. the $(k^*,1)$ domain), the repeated application of $\calC$ decreases values without bringing them negative because of \autoref{lem:cconvex}, so values are eventually all brought to the $[0,k^*]$ domain. There, our derivative is strictly increasing because of convexity, and it stays between 1 and 0.
\begin{equation}
    \calC'(0) = \frac{(a_1)^2}{\brak{\sigma(z)^2}} \geq 0,
\end{equation}
\begin{equation}
    \calC'(k^*) = \frac{1}{1-k^*} \int_{k^*}^1 \calC'(k^*) dk < \frac{1}{1-k^*} \int_{k^*}^1 \calC'(k)dk = \frac{1-k^*}{1-k^*} = 1.
\end{equation}
    In the $[0,k^*]$ domain, these bounds on the derivative imply bounds on the range.
\begin{equation}
    k < k^* - \int_{k}^{k^*} dk < \calC(k) = k^* - \int_{k}^{k^*} \calC'(j) \, dj \leq k^*.
\end{equation}
    Thus, $\calC$ is a contracting map in the $[0,k^*]$ domain, and by the Banach fixed-point theorem, has an attractive fixed point reached by repeated application.
\end{proof}

\subsection{Proof of \autoref{thm:mix}}

\begin{proof}
    This theorem of course inherits the structure of the overall probability distribution of $P(\zl{\ell} \mid \calD)$ from \autoref{thm:roberts}. 
    The dataset containing no scalar multiples is equivalent to the condition that initial covariances $\igil{\alpha_1 \alpha_2}{1}$ are not 1 or $-1$, and so \autoref{lem:fixpoint} applies for non-affine activation functions. We will handle affine, nonlinear functions separately. At each layer, the change in covariance is determined by applying the function $\calC$, and the covariance approaches its globally attractive fixed point. That fixed point is 0 exactly when $\calC(0)=0$, and from \eqref{eq:cpoly}, that condition is equivalent to the activation function having zero mean under the Gaussian measure.
\begin{equation}
    \calC(0)=\frac{(a_0)^2 }{\brak{\sigma(z)^2}} = \frac{\brak{\sigma(z)}^2}{\brak{\sigma(z)^2}}.
\end{equation}
    At a zero fixed point, the derivative must be less than 1 by strict convexity (\autoref{lem:cconvex}), and so the covariance approaches the fixed point exponentially fast when it is close enough. Thus proving \autoref{thm:mix} for nonaffine functions.

    For affine, nonlinear functions $\sigma(z) = az + b$, the Hermite decomposition is simple enough to compute exactly. From the decomposition we see that their mean under the Gaussian measure is always positive, and repeated applications of $\calC$ always brings initial values to a fixed point at 1.
\begin{equation}
        \calC(k) = \frac{a^2 k + b^2}{\brak{\sigma(z)^2}}.
\end{equation}
    Thus, this proves \autoref{thm:mix} for affine, nonlinear functions.
\end{proof}

\section{Conclusion}

    Under the perturbative expansion and approximation from \citet{roberts22_pdlt}, we have established that randomly initialized neural networks behave like random functions and have nearly independent outputs exactly when the activation function $\sigma$ is nonlinear and has zero mean under the Gaussian measure,
\begin{equation*}
    \brak{\sigma(z)}=0.
\end{equation*}
    We propose these very wide neural networks as a promising candidate for the computational no-coincidence conjecture. Compared to random reversible circuits, these neural networks are closer to what we ultimately care about in practice. However, a remaining advantage of random reversible circuits is that, for polynomial-size circuits, their outputs have exponentially small dependence\footnotemark{} \citep{gay25_rrc} while our results only establish polynomially small dependence. Achieving exponentially small dependence between the outputs of a neural network may require exponentially large width because of the $O(1/n)$ term in \eqref{eq:pzlmd}. For completeness, let us now state a version of the no-coincidence conjecture for neural networks.
    \footnotetext{That is, the probability distribution over the outputs of a random reversible circuit has exponentially small total variation distance from the uniform distribution.}

\begin{conjecture}[The neural network computational no-coincidence conjecture]
    For a randomly initialized neural network $C : \bbR^{2n} \to \bbR^n$ with hidden dimension $2n$, depth $\ell = \Theta(\log(n))$, Gaussian random weights, and tanh nonlinearities, let $P(C)$ be the property that for none of the inputs in the dataset $x \in \braces{-1,1}^{2n}$, the output $C(x)$ is all negative. There exists a polynomial-time verification algorithm $V$ that receives as input the neural network $C$ and an advice string $\pi$ such that for all $C$ where $P(C)$ is true, there exists $\pi$ with length polynomial in the size of $C$ such that $V(C,\pi)=1$, and for 99\% of random neural networks $C$, no such $\pi$ exists.
\end{conjecture}

\parbreak
    We are agnostic on the truth or falsity of this conjecture, but we put it forward as the ``correct" analogue of the Alignment Research Center's computational no-coincidence conjecture for neural networks and hope that it inspires followup work.

\section{Acknowledgements}
    We are grateful to Cody Rushing for helpful comments on drafts and to Dmitry Vaintrob for pointing us to relevant parts of \citet{roberts22_pdlt}.
\bibliography{bib}

@misc{hanin19,
      title={Deep ReLU Networks Have Surprisingly Few Activation Patterns}, 
      author={Boris Hanin and David Rolnick},
      year={2019},
      eprint={1906.00904},
      archivePrefix={arXiv},
      primaryClass={stat.ML},
      url={https://arxiv.org/abs/1906.00904}, 
}

@misc{depalma19,
      title={Random deep neural networks are biased towards simple functions}, 
      author={Giacomo De Palma and Bobak Toussi Kiani and Seth Lloyd},
      year={2019},
      eprint={1812.10156},
      archivePrefix={arXiv},
      primaryClass={stat.ML},
      url={https://arxiv.org/abs/1812.10156}, 
}

@misc{teney24,
      title={Neural Redshift: Random Networks are not Random Functions}, 
      author={Damien Teney and Armand Nicolicioiu and Valentin Hartmann and Ehsan Abbasnejad},
      year={2024},
      eprint={2403.02241},
      archivePrefix={arXiv},
      primaryClass={cs.LG},
      url={https://arxiv.org/abs/2403.02241}, 
}

@misc{poole16,
      title={Exponential expressivity in deep neural networks through transient chaos}, 
      author={Ben Poole and Subhaneil Lahiri and Maithra Raghu and Jascha Sohl-Dickstein and Surya Ganguli},
      year={2016},
      eprint={1606.05340},
      archivePrefix={arXiv},
      primaryClass={stat.ML},
      url={https://arxiv.org/abs/1606.05340}, 
}

@misc{lee18,
      title={Deep Neural Networks as Gaussian Processes}, 
      author={Jaehoon Lee and Yasaman Bahri and Roman Novak and Samuel S. Schoenholz and Jeffrey Pennington and Jascha Sohl-Dickstein},
      year={2018},
      eprint={1711.00165},
      archivePrefix={arXiv},
      primaryClass={stat.ML},
      url={https://arxiv.org/abs/1711.00165}, 
}

@misc{goldwasser24_backdoors,
      title={Planting Undetectable Backdoors in Machine Learning Models}, 
      author={Shafi Goldwasser and Michael P. Kim and Vinod Vaikuntanathan and Or Zamir},
      year={2024},
      eprint={2204.06974},
      archivePrefix={arXiv},
      primaryClass={cs.LG},
      url={https://arxiv.org/abs/2204.06974}, 
}

@article{anthropic25_bio,
  author={Lindsey, Jack and Gurnee, Wes and Ameisen, Emmanuel and Chen, Brian and Pearce, Adam and Turner, Nicholas L. and Citro, Craig and Abrahams, David and Carter, Shan and Hosmer, Basil and Marcus, Jonathan and Sklar, Michael and Templeton, Adly and Bricken, Trenton and McDougall, Callum and Cunningham, Hoagy and Henighan, Thomas and Jermyn, Adam and Jones, Andy and Persic, Andrew and Qi, Zhenyi and Thompson, T. Ben and Zimmerman, Sam and Rivoire, Kelley and Conerly, Thomas and Olah, Chris and Batson, Joshua},
  title={On the Biology of a Large Language Model},
  journal={Transformer Circuits Thread},
  year={2025},
  url={https://transformer-circuits.pub/2025/attribution-graphs/biology.html}
}

@book{roberts22_pdlt,
   title={The Principles of Deep Learning Theory: An Effective Theory Approach to Understanding Neural Networks},
   ISBN={9781316519332},
   url={http://dx.doi.org/10.1017/9781009023405},
   DOI={10.1017/9781009023405},
   publisher={Cambridge University Press},
   author={Roberts, Daniel A. and Yaida, Sho and Hanin, Boris},
   year={2022},
   month=may 
}

@book{bateman1955,
  author       = {Bateman, Harry and Erd{\'e}lyi, Arthur},
  address      = {New York},
  organization = {Bateman Manuscript Project.},
  title        = {Higher Transcendental Functions},
  volume       = {2},
  publisher    = {McGraw-Hill Book Company},
  year         = {1953},
  month        = aug,
  note         = {Based on notes left by Harry Bateman and compiled by the staff of the Bateman Manuscript Project},
  noturl       = {https://authors.library.caltech.edu/records/cnd32-h9x80},
}

@article{kibble1945,
	author = {Kibble, W. F.},
	address = {Cambridge, UK},
	copyright = {Copyright © Cambridge Philosophical Society 1945},
	issn = {0305-0041},
	journal = {Mathematical Proceedings of the Cambridge Philosophical Society},
	language = {eng},
	number = {1},
	pages = {12-15},
	publisher = {Cambridge University Press},
	title = {An extension of a theorem of Mehler's on Hermite polynomials},
	volume = {41},
	year = {1945},
}

@misc{neyman25_nocoincidence,
	title={A computational no-coincidence principle},
	author={Eric Neyman},
	year={2025},
	url = {https://www.alignment.org/blog/a-computational-no-coincidence-principle/},
}

@misc{gay25_rrc,
      title={Pseudorandomness Properties of Random Reversible Circuits}, 
      author={William Gay and William He and Nicholas Kocurek and Ryan O'Donnell},
      year={2025},
      eprint={2502.07159},
      archivePrefix={arXiv},
      primaryClass={cs.CR},
      url={https://arxiv.org/abs/2502.07159}, 
}

@article{anthropic24_mono,
   title={Scaling Monosemanticity: Extracting Interpretable Features from Claude 3 Sonnet},
   author={Templeton, Adly and Conerly, Tom and Marcus, Jonathan and Lindsey, Jack and Bricken, Trenton and Chen, Brian and Pearce, Adam and Citro, Craig and Ameisen, Emmanuel and Jones, Andy and Cunningham, Hoagy and Turner, Nicholas L and McDougall, Callum and MacDiarmid, Monte and Freeman, C. Daniel and Sumers, Theodore R. and Rees, Edward and Batson, Joshua and Jermyn, Adam and Carter, Shan and Olah, Chris and Henighan, Tom},
   year={2024},
   journal={Transformer Circuits Thread},
   url={https://transformer-circuits.pub/2024/scaling-monosemanticity/index.html}
}

@misc{goldowskydill25_deception,
      title={Detecting Strategic Deception Using Linear Probes}, 
      author={Nicholas Goldowsky-Dill and Bilal Chughtai and Stefan Heimersheim and Marius Hobbhahn},
      year={2025},
      eprint={2502.03407},
      archivePrefix={arXiv},
      primaryClass={cs.LG},
      url={https://arxiv.org/abs/2502.03407}, 
}

@misc{turner24_steering,
      title={Steering Language Models With Activation Engineering}, 
      author={Alexander Matt Turner and Lisa Thiergart and Gavin Leech and David Udell and Juan J. Vazquez and Ulisse Mini and Monte MacDiarmid},
      year={2024},
      eprint={2308.10248},
      archivePrefix={arXiv},
      primaryClass={cs.CL},
      url={https://arxiv.org/abs/2308.10248}, 
}

@misc{cloud24_gradrouting,
      title={Gradient Routing: Masking Gradients to Localize Computation in Neural Networks}, 
      author={Alex Cloud and Jacob Goldman-Wetzler and Evžen Wybitul and Joseph Miller and Alexander Matt Turner},
      year={2024},
      eprint={2410.04332},
      archivePrefix={arXiv},
      primaryClass={cs.LG},
      url={https://arxiv.org/abs/2410.04332}, 
}

@misc{zou24_circuitbreakers,
      title={Improving Alignment and Robustness with Circuit Breakers}, 
      author={Andy Zou and Long Phan and Justin Wang and Derek Duenas and Maxwell Lin and Maksym Andriushchenko and Rowan Wang and Zico Kolter and Matt Fredrikson and Dan Hendrycks},
      year={2024},
      eprint={2406.04313},
      archivePrefix={arXiv},
      primaryClass={cs.LG},
      url={https://arxiv.org/abs/2406.04313}, 
}

@InProceedings{chrisiano25_backdoors,
  author =	{Christiano, Paul and Hilton, Jacob and Lecomte, Victor and Xu, Mark},
  title =	{{Backdoor Defense, Learnability and Obfuscation}},
  booktitle =	{16th Innovations in Theoretical Computer Science Conference (ITCS 2025)},
  pages =	{38:1--38:21},
  series =	{Leibniz International Proceedings in Informatics (LIPIcs)},
  ISBN =	{978-3-95977-361-4},
  ISSN =	{1868-8969},
  year =	{2025},
  volume =	{325},
  editor =	{Meka, Raghu},
  publisher =	{Schloss Dagstuhl -- Leibniz-Zentrum f{\"u}r Informatik},
  address =	{Dagstuhl, Germany},
  URL =		{https://drops.dagstuhl.de/entities/document/10.4230/LIPIcs.ITCS.2025.38},
  URN =		{urn:nbn:de:0030-drops-226662},
  doi =		{10.4230/LIPIcs.ITCS.2025.38}
}

@misc{he15,
      title={Delving Deep into Rectifiers: Surpassing Human-Level Performance on ImageNet Classification}, 
      author={Kaiming He and Xiangyu Zhang and Shaoqing Ren and Jian Sun},
      year={2015},
      eprint={1502.01852},
      archivePrefix={arXiv},
      primaryClass={cs.CV},
      url={https://arxiv.org/abs/1502.01852}, 
}

\clearpage
\appendix

\section{Simplified Proof of \autoref{thm:roberts}}\label{sec:robertsproof}
    In this section we will establish the zeroth order pertubative expansion of $P(\zl{\ell} \mid \calD)$, where we will use big-O notation to hide $O(1/\n)$ terms in addition to the $O(1/\n^2)$ terms hidden in \autoref{thm:roberts}.
    Since this proof involves heavy use of inverses, we will use raised indices as shorthand for indexing into the inverse of a matrix.
\begin{equation}\label{eq:Gnotation}
\begin{gathered}
    \IGIL{i j}{\ell} = \ti{((\igl{\ell})^{-1})}{i j} \\
    \sum_{j} \IGIL{i j}{\ell} \igil{j k}{\ell} = \delta_{i k}
\end{gathered}
\end{equation}
    Here, $\delta_{i k}$ is the Kronecker delta.
\begin{proof} 
    This proof will go by induction on the depth of the neural network.

\parbreak
    \textit{Base case:} At the first layer, the distribution of preactivations is exactly Gaussian because the weights and biases, $\wil{i j}{1}$ and $\biasil{i}{1}$, are a multivariate Gaussian by definition, and each activation $\zial{i}{\alpha}{1}$ is a linear combination of the weights parameterized by the input $\ti{x}{\alpha}$. The covariances are as follows:
\begin{equation}
\begin{aligned}
    \E[\zial{i_1}{\alpha}{1} \zial{i_2}{\alpha_2}{1}]
        &= \E[(\biasi{i_1} + \sum_j \wi{i_1 j} \, \tia{x}{j}{\alpha_1})(\biasi{i_2} + \sum_j \wi{i_2 j} \, \tia{x}{j}{\alpha_2})] \\
        &= \delta_{i_1 i_2} (C_b + \frac{C_W}{n} \sum_j \tia{x}{j}{\alpha_1} \tia{x}{j}{\alpha_2}) \\
        &= \delta_{i_1 i_2} \igil{\alpha_1 \alpha_2}{1}
\end{aligned}
\end{equation}
    Thus, the distribution of first layer preactivations satisfies the inductive hypothesis. Using the notation from \eqref{eq:Gnotation} for matrix inverses, we can describe the distribution.
\begin{equation}
    P(\zl{1} \mid \calD) \propto \exp\le(- \frac{1}{2} \sum_{\alpha_1 \alpha_2 \in \calD} \IGIL{\alpha_1 \alpha_2}{1} \sum_{i=1}^{\nl{\ell-1}} \zial{i}{\alpha_1}{1} \zial{i}{\alpha_2}{1} \ri)
\end{equation}

\parbreak
    \textit{Inductive step:} We will look at the following recurrence to prove that $P(\zl{\ell} \mid \calD)$ is distributed as desired.
\begin{equation}\label{eq:pzl}
    P(\zl{\ell} \mid \calD) = \int P(\zl{\ell-1} \mid \calD) P(\zl{\ell} \mid \zl{\ell-1}) \prod_{i, \alpha} d\zial{i}{\alpha}{\ell-1}
\end{equation}
    However, before we start working with nearly Gaussian distributions, let's take a brief digression to introduce a useful identity for evaluating expectations. Let $\eps = 1 / n$.
\begin{equation}\label{eq:avgf}
\begin{aligned}
    \E[f(\zial{i_1}{\alpha_1}{\ell}, \ldots, \zial{i_m}{\alpha_m}{\ell})]
    &= \int \frac{\exp(- \frac{1}{2} \sum_{i, \alpha} \IGIL{\alpha_1 \alpha_2}{\ell} \zia{i_1}{\alpha_1} \zia{i_2}{\alpha_2} + O(\eps))}
    {\le(\int\exp(- \frac{1}{2} \sum_{i, \alpha} \IGIL{\alpha_1 \alpha_2}{\ell} \zia{i_1}{\alpha_1} \zia{i_2}{\alpha_2} + O(\eps)) \prod_{i, \alpha} \tia{dz}{i}{\alpha} \ri)}
    f(z) \prod_{i, \alpha} \tia{dz}{i}{\alpha} \\
    &= \int \frac{1}{\sqrt{|2 \pi \igl{\ell}|^\n}} \exp\le(- \frac{1}{2} \sum_{i, \alpha} \IGIL{\alpha_1 \alpha_2}{\ell} \zia{i_1}{\alpha_1} \zia{i_2}{\alpha_2}\ri) f(z) (1 + O(\eps)) \\
\end{aligned}
\end{equation}
    In the second equality, we Taylor expanded the expression around $\eps = 0$. We have also dropped the superscript $\tl{}{\ell}$ to reduce clutter. When the indicies $i_1, \ldots i_m$ are equal, we can marginalize over the unused indices and use the notation from \eqref{eq:braknotation} for Gaussian expectations.
\begin{equation}
\begin{aligned}
    \E[f(\zial{i}{\alpha_1}{\ell}, \ldots, \zial{i}{\alpha_m}{\ell})] = \brak{f(\ti{z}{\alpha_1}, \ldots, \ti{z}{\alpha_m})}_{\igl{\ell}} + O(\eps)
\end{aligned}
\end{equation}
    Returning back to \eqref{eq:pzl}, the $P(\zl{\ell-1} \mid \calD)$ term will be determined by our inductive hypotheses, so let's consider the other term, $P(\zl{\ell} \mid \zl{\ell-1})$. Similar to what happened in the first layer with fixed layer inputs $\zl{\ell-1}$, this distribution is exactly Gaussian.
\begin{equation}
\begin{gathered}
    P(\zl{\ell} \mid \zl{\ell-1}) = 
    \frac{1}{\sqrt{|2 \pi \hgl{\ell}|^n}} \exp\le(- \frac{1}{2} \sum_{i \alpha} \HGIL{\alpha_1 \alpha_2}{\ell} \zial{i_1}{\alpha_1}{\ell} \zial{i_2}{\alpha_2}{\ell}\ri) \\
    \hgil{\alpha_1 \alpha_2}{\ell} = \cb + \frac{\cw}{\nl{\ell-1}} \sum_{i=1}^{\nl{\ell-1}} \sial{i}{\alpha_1}{\ell-1} \sial{i}{\alpha_2}{\ell-1}
\end{gathered}
\end{equation}
    Here, we have introduced $\hgil{\alpha_1 \alpha_2}{\ell}$, a random variable that depends on $\zl{\ell-1}$. Let's split $\hgil{\alpha_1 \alpha_2}{\ell}$ into its expectation and deviations. 
\begin{equation}
\begin{gathered}
    \gil{\alpha_1 \alpha_2}{\ell} = \E[\hgil{\alpha_1 \alpha_2}{\ell}] = C_b + \frac{C_W}{\nl{\ell-1}} \sum_{i=1}^{\nl{\ell-1}} \E[\sial{i}{\alpha_1}{\ell-1} \sial{i}{\alpha_2}{\ell-1}]  \\
    \dgil{\alpha_1 \alpha_2}{\ell} = \hgil{\alpha_1 \alpha_2}{\ell} - \gil{\alpha_1 \alpha_2}{\ell}
\end{gathered}
\end{equation}
    One reason this split will be useful is because the expectation is close to the desired covariance from \eqref{eq:gildef}. 
\begin{equation}\label{eq:giligil}
    \gil{\alpha_1 \alpha_2}{\ell} = \igil{\alpha_1 \alpha_2}{\ell} + O\le(\eps\ri)
\end{equation}
    This can be easily proved by induction. Notice that the only difference between the two terms is the expectation $\E[\sial{i}{\alpha_1}{\ell-1} \sial{i}{\alpha_2}{\ell-1}]$ instead of the Gaussian expectation $\brak{\sigma(\zi{\alpha_1})\sigma(\zi{\alpha_2})}_{\igl{\ell-1}}$. At layer 2 when the preactivations are exactly Gaussian, the terms $\gil{\alpha_1 \alpha_2}{2}$ and $\igil{\alpha_1 \alpha_2}{2}$ are trivially equal, and in later layers we only need to apply \eqref{eq:avgf} to bound the difference between the nearly-Gaussian and Gaussian expectations. 
    A second reason this split will be useful is because the variance of the deviations is quite small.
\begin{equation}
\begin{aligned}
    \E[\dgil{\alpha_1 \alpha_2}{\ell} \dgil{\alpha_3 \alpha_4}{\ell}] 
    &= \le(\frac{\cw}{\nl{\ell-1}}\ri)^2 \sum_{j,k=1}^{\nl{\ell-1}} \E[(\sia{j}{\alpha_1} \sia{j}{\alpha_2} - \E[\sia{j}{\alpha_1} \sia{j}{\alpha_2}])(\sia{k}{\alpha_3} \sia{k}{\alpha_4} - \E[\sia{k}{\alpha_3} \sia{k}{\alpha_4}])] \\
\end{aligned}
\end{equation}
    In the part of this equation inside the sum, we can use \eqref{eq:avgf} and find that the $\nl{\ell-1}^2$ off-diagonal terms are $O(\eps)$ and the $\nl{\ell-1}$ diagonal terms are the following:
\begin{equation}
    \brak{\sigma(\zi{\alpha_1}) \sigma(\zi{\alpha_2}) \sigma(\zi{\alpha_3}) \sigma(\zi{\alpha_4})}_{\igl{\ell-1}} - \brak{\sigma(\zi{\alpha_1}) \sigma(\zi{\alpha_2})}_{\igl{\ell-1}} \brak{\sigma(\zi{\alpha_3}) \sigma(\zi{\alpha_4})}_{\igl{\ell-1}} + O(\eps)
\end{equation}
    These cancel out nicely with the leading coefficient $(1/\nl{\ell-1})^2$, so the overall variance is $O(\eps)$.
\begin{equation}
    \E[\dgil{\alpha_1 \alpha_2}{\ell} \dgil{\alpha_3 \alpha_4}{\ell}] = O(\eps)
\end{equation}
    By a similar argument, all the higher moments of $\dg$ are also $O(\eps)$.
    Although the covariances $\hgil{\alpha_1 \alpha_2}{\ell}$ are straightforward to compute, inverting the matrix is slightly harder. We will need the following common identity for inverses of perturbed matrices.
\begin{equation}\label{eq:perturbationinverse}
\begin{aligned}
    \HGIL{\alpha_1 \alpha_2}{\ell} &= \ti{((\gl{\ell} + \dgl{\ell})^{-1})}{\alpha_1 \alpha_2}  \\
    &= \GIL{\alpha_1 \alpha_2}{\ell} - \sum_{\beta} \GIL{\alpha_1 \beta_1}{\ell} \dgil{\beta_1 \beta_2}{\ell} \GIL{\beta_2 \alpha_2}{\ell} + O(\Delta^2)
\end{aligned}
\end{equation}
    At this point we have the tools necessary to expand and simplify \eqref{eq:pzl} to prove our inductive hypothesis. Before doing so however, let's recall the following properties of $\dg$ that will be essential for the simplification.
\begin{equation}\label{eq:dgproperties}
\begin{aligned}
    \E[\dgil{\alpha_1 \alpha_2}{\ell}] = 0 \\
    \E[O(\Delta^2)] = O(\eps)
\end{aligned}
\end{equation}
    The expansion of \eqref{eq:pzl} will also create really long equations, so let's define the following shorthand.
\begin{equation}\label{eq:skshorthand}
\begin{aligned}
    S(z) &= - \frac{1}{2} \sum_{i, \alpha} \GIL{\alpha_1 \alpha_2}{\ell} \zia{i_1}{\alpha_1} \zia{i_2}{\alpha_2} \\
    \Delta H_\alpha &= \frac{1}{2}\sum_{\beta} \GIL{\alpha_1 \beta_1}{\ell} \dgil{\beta_1 \beta_2}{\ell} \GIL{\beta_2 \alpha_2}{\ell}
\end{aligned}
\end{equation}
    Now we are ready to put it all together.
\begin{equation}
\begin{aligned}
    P(\zl{\ell} \mid \calD) &= \int P(\zl{\ell-1} \mid \calD) P(\zl{\ell} \mid \zl{\ell-1}) \prod_{i, \alpha} \tial{dz}{i}{\alpha}{\ell-1} \\
    &= \E[P(\zl{\ell} \mid \zl{\ell-1})] \\
    &= \E\le[\frac{
        \exp\le(-\frac{1}{2} \sum_{\alpha_1, \alpha_2 \in \calD} \HGIL{\alpha_1 \alpha_2}{\ell} \sum_{i=1}^{\nl{\ell}} \zial{i}{\alpha_1}{\ell} \zial{i}{\alpha_2}{\ell} \ri)
    }{   
        \le(\int \exp\le(-\frac{1}{2} \sum_{\alpha_1, \alpha_2 \in \calD} \HGIL{\alpha_1 \alpha_2}{\ell} \sum_{i=1}^{\nl{\ell}} \zial{i}{\alpha_1}{\ell} \zial{i}{\alpha_2}{\ell}\ri) \prod_{i, \alpha} \tial{dz}{i}{\alpha}{\ell} \ri)
    }\ri] \\
    &= \E\le[\frac{
        \exp\le(S(z) + \sum_{i, \alpha}(\Delta H_\alpha + O(\Delta^2)) \zia{i}{\alpha_1} \zia{i}{\alpha_2}\ri)
    }{
        \le(\int \exp\le(S(z) + \sum_{i, \alpha}(\Delta H_\alpha + O(\Delta^2)) \zia{i}{\alpha_1} \zia{i}{\alpha_2}\ri) \prod_{i, \alpha} \tia{dz}{i}{\alpha}\ri)
    }\ri] \\
    & =\E\Bigg[\begin{aligned}[t]
        & \frac{\exp(S(z))}{\int \exp(S(z)) \prod_{i, \alpha} \tia{dz}{i}{\alpha}} \\
        &\times\Bigg(1 \begin{aligned}[t] &+ \sum_{i,\alpha} \Delta H_\alpha \zia{i}{\alpha_1} \zia{i}{\alpha_2} + O(\Delta^2) \\
        &- \frac{\int (\sum_{\alpha} \Delta H_\alpha + O(\Delta^2)) \zia{i}{\alpha_1} \zia{i}{\alpha_2} \exp(S(z)) \prod_{i, \alpha}\tia{dz}{i}{\alpha}}{\int \exp(S(z)) \prod_{i, \alpha} \tia{dz}{i}{\alpha}}
    \Bigg)\Bigg]\end{aligned}\end{aligned} \\
    &= \frac{\exp(S(z))}{\int \exp(S(z)) \prod_{i, \alpha} \tial{dz}{i}{\alpha}{\ell}} (1 + O(\eps)) \\
    &\propto \exp(S(z) + O(\eps)) \\
    &\propto \exp\le(-\frac{1}{2} \sum_{i, \alpha} \IGIL{\alpha_1 \alpha_2}{\ell} \zial{i}{\alpha_1}{\ell} \zial{i}{\alpha_2}{\ell} + O(\eps) \ri)
\end{aligned}
\end{equation}
In the fourth equality we inserted \eqref{eq:perturbationinverse} with notation from \eqref{eq:skshorthand}, in the fifth we Taylor expanded around $\eps=0$, in the sixth we applied \eqref{eq:dgproperties}, the useful properties of $\dg$, in the seventh we took the logorithm and Taylor expanded to move the $O(\eps)$ term into the exponent, and in the last equality we applied \eqref{eq:giligil}, the difference bound between $\hgil{\alpha_1 \alpha_2}{\ell}$ and $\igil{\alpha_1 \alpha_2}{\ell}$, combined with \eqref{eq:perturbationinverse} to bound the difference of their inverses. Now that we have proven that the probability distribution is of the desired form, it's simple to check that the covariances are as desired as well. We can even use \eqref{eq:avgf} again.
\begin{equation}
\begin{aligned}
    \E[\zial{i_1}{\alpha_1}{\ell} \zial{i_2}{\alpha_2}{\ell}] &= \delta_{i_1 i_2} \brak{\zil{\alpha_1}{\ell} \zil{\alpha_2}{\ell}}_{\igl{\ell}} + O(\eps) \\
    &= \delta_{i_1 i_2} \igil{\alpha_1 \alpha_2}{\ell} + O(\eps)
\end{aligned}
\end{equation}
Thus, we have completed the proof that the probability distribution $P(\zl{\ell} \mid \calD)$ and its covariances are of the desired form.
\end{proof}

\section{Derivation of Equation \ref{eq:mehler}} \label{sec:mehlerproof}

Let $\gamma_\Sigma (u_1, u_2)$ and $\gamma_{I} (u_1, u_2)$ represent measures on the 2D unit-variance Gaussian when the covariance is $k$ and $0$.
\begin{equation}
    \begin{aligned}
\gamma_\Sigma (u_1, u_2) &= \frac{1}{2\pi\sqrt{1-k^2}} \exp\left(\frac{2 u_1 u_2 k - u_1^2 - u_2^2}{2(1-k^2)}\right) \\
\gamma_{I} (u_1, u_2) &= \frac{1}{2\pi} \exp\left(\frac{-u_1^2 - u_2^2}{2}\right)
    \end{aligned}
\end{equation}
The following is a standard formulation of Mehler's equation in terms of Gaussian measures \citep{kibble1945}.
\begin{equation}
\gamma_{I} (u_1, u_2) \sum_{n=0}^{\infty} \frac{k^n}{n!} \He_n (u_1) \He_n (u_2) = \gamma_\Sigma (u_1, u_2)
\end{equation}
Using Mehler's equation, we can prove \eqref{eq:mehler}, the product of two Hermite polynomials under correlated Gaussian weighting.
\begin{equation}
\begin{aligned}
    & \brak{\He_n (z_1) \He_m (z_2)}_\Sigma, \quad \Sigma = \begin{pmatrix} 1 & k \\ k & 1 \end{pmatrix} \\
    =& \iint \He_n (z_1) \He_m (z_2) \gamma_\Sigma (z_1, z_2) \, dz_1 \, dz_2 \\
    =& \iint \He_n (z_1) \He_m (z_2) \left(\sum_{i=0}^{\infty} \frac{k^i}{i!} \He_i (z_1) \He_i (z_2)\right) \gamma_{I} (z_1, z_2) \, dz_1 \, dz_2 \\
    =& \brak*{\He_n (z_1) \He_m (z_2) \left(\sum_{i=0}^{\infty} \frac{k^i}{i!} \He_i (z_1) \He_i (z_2)\right)} \\
    =& \brak*{\He_m (z_2) \left(\sum_{i=0}^{\infty} \brak*{\frac{k^i}{i!} \He_n(z_1) \He_i (z_1)} \He_i (z_2) \right)} \\
    =& \brak*{\He_m (z_2) \brak*{\frac{k^n}{n!}\He_n (z_1)^2} \He_n (z_2)} \\
    =& \brak{\He_m (z_2) \He_n (z_2) k^n} \\
    =& \,\delta_{nm} \, n! \, k^n
\end{aligned}
\end{equation}

\end{document}